\documentclass{article}

\newif\ifpreprint

\usepackage[nonatbib,preprint]{neurips_2020}
\preprinttrue 

\usepackage[utf8]{inputenc} 
\usepackage[T1]{fontenc}    
\usepackage{hyperref}       
\usepackage{url}            
\usepackage{booktabs}       
\usepackage{amsfonts}       
\usepackage{amsmath}
\usepackage{amssymb}
\usepackage{graphicx}
\usepackage{todonotes}
\usepackage{amsthm}
\usepackage{nicefrac}       
\usepackage{microtype}      
\usepackage{subcaption}
\usepackage{textcomp}
\usepackage{siunitx}
\usepackage{wrapfig}
\newcommand{\tr}[1]{\operatorname{Tr}\left\{#1\right\}}

\DeclareMathOperator{\sinc}{sinc}
\newcommand{\Rn}{\mathbb{R}}
\newcommand{\vf}{\mathbf}
\newcommand{\intbounds}{\int_{-1}^{1-2\min\{|t|,1\}}}
\title{Sparse Gaussian Processes via Parametric Families of Compactly-supported Kernels}

\author{%
Jarred Barber \\
Charles River Analytics \\
Cambridge, MA 02138 \\
\texttt{jarred.barber@gmail.com}
}
\newtheorem{theorem}{Theorem}
\newtheorem{lemma}{Lemma}
\begin{document}
\maketitle
\begin{abstract}
Gaussian processes are powerful models for probabilistic machine learning, but are limited in application by their $O(N^3)$ inference complexity. We propose a method for deriving parametric families of kernel functions with compact spatial support, which yield naturally sparse kernel matrices and enable fast Gaussian process inference via sparse linear algebra. These families generalize known compactly-supported kernel functions, such as the Wendland polynomials.
The parameters of this family of kernels can be learned from data using maximum likelihood estimation. Alternatively, we can quickly compute compact approximations of a target kernel using convex optimization. 
We demonstrate that these approximations incur minimal error over the exact models when modeling data drawn directly from a target GP, and can out-perform the traditional GP kernels on real-world signal reconstruction tasks, while exhibiting sub-quadratic inference complexity.
\end{abstract}

\section{Introduction}
In recent years, Gaussian processes (GPs) have become an increasingly popular class of models in machine learning due to their simplicity of implementation, flexibility, and ability to perform exact Bayesian posterior inference on observed data. They are particularly popular in time series modeling\cite{roberts2013gaussian}. Recently, GPs were proposed as models for Bayesian signal processing with the introduction of the sinc kernel\cite{tobar2019band} and recent work on spectral analysis with GPs\cite{pmlr-v84-ambrogioni18a}.

An issue that plagues wider adoption of GPs is computational complexity. General exact inference methods on $N$ data points, as well as the computation of the data likelihood gradient (for parameter learning via maximum likelihood) requires the inversion of an $N\times N$ covariance matrix. Barring any special structure in the matrices, this requires $O(N^3)$ time and $O(N^2)$ space. This makes them ill suited for very large machine learning tasks with hundreds of thousands or millions of data points. Even on smaller problems, $O(N^3)$ inference complexity limits the utility of GPs for applications with latency requirements or compute limitations.

There has been a variety of work on scaling Gaussian processes to large datasets\cite{liu2020gaussian}, typically through approximations. These include methods that subsample data\cite{nips2005_matching_pursuit}, or find low-rank approximations of the kernel matrix\cite{kissgp}. 
Recent work\cite{gardner2018gpytorch} has focused on exploiting GPU accelerated black-box matrix-vector multiply (MVM) operations to compute matrix inverses with iterative algorithms, such as the conjugate gradient algorithm. These results have been extended to exact GP training and inference on large datasets\cite{NIPS2019_9606}; by computing the kernel dynamically, they avoid the $O(N^2)$ storage requirement to enable the use of GPU-accelerated MVMs, but do not directly decrease the $O(N^2)$ MVM time complexity.

A little-explored area has been in choices of GP kernels with compact support that naturally yield \emph{sparse} kernel matrices\cite{wendland1995piecewise}\cite{melkumyan2009sparse}, which we will abbreviate here as "compact kernels". These kernels can yield very fast matrix-vector multiply operations, on the order of the number of \emph{non-zero} entries of the kernel matrix. If a dataset yields good sparsity structure, this can be exploited by black-box MVM methods for $O(N^2)$ or better covariance matrix inversion. However, there are only a small number of fixed kernels of this class that have been identified in the literature; they lack specific hyperparameters that can be tuned to data, and are qualitatively very similar to each other.

This paper generalizes and extends these to larger families of functions parameterized by positive definite matrices. We show that given a set of $M$ basis functions, we can generate a space of compactly supported positive definite kernels with dimensionality up to $\frac{M(M+1)}{2}$ through the evaluation of a particular integral. We then construct two concrete examples of these kernel families using polynomial and Fourier basis functions.

These functions can be fit directly to data using maximum likelihood estimation, but we also show that they can be used to approximate existing kernels. We derive an efficient algorithm for computing a compact approximation of a fixed kernel using convex optimization.

Finally, we perform experiments to demonstrate the utility of our constructed kernels. We show that sparse linear algebra can provide substantial inference speedups over na\"ive implementations of GP regression (the exact speedup depending on the structure of the data). We also compare a set of common non-compact kernels to their compact approximations, derived using our convex optimization approach, to simulated data (drawn from a known GP) as well as real data (speech data from the TIMIT\cite{garofolo1993darpa} dataset). We then compare common non-compact kernels and previously identified compact kernels with our parametric kernels by directly maximizing likelihood on the TIMIT dataset, using the GPyTorch\cite{gardner2018gpytorch} framework. We find that the flexibility of our parameterization allows for better performance both in terms of likelihood as well as reconstruction RMSE.

\begin{figure}[t!]
    \centering
    \begin{subfigure}[t]{0.3\linewidth}
         \centering
         \includegraphics[width=\linewidth]{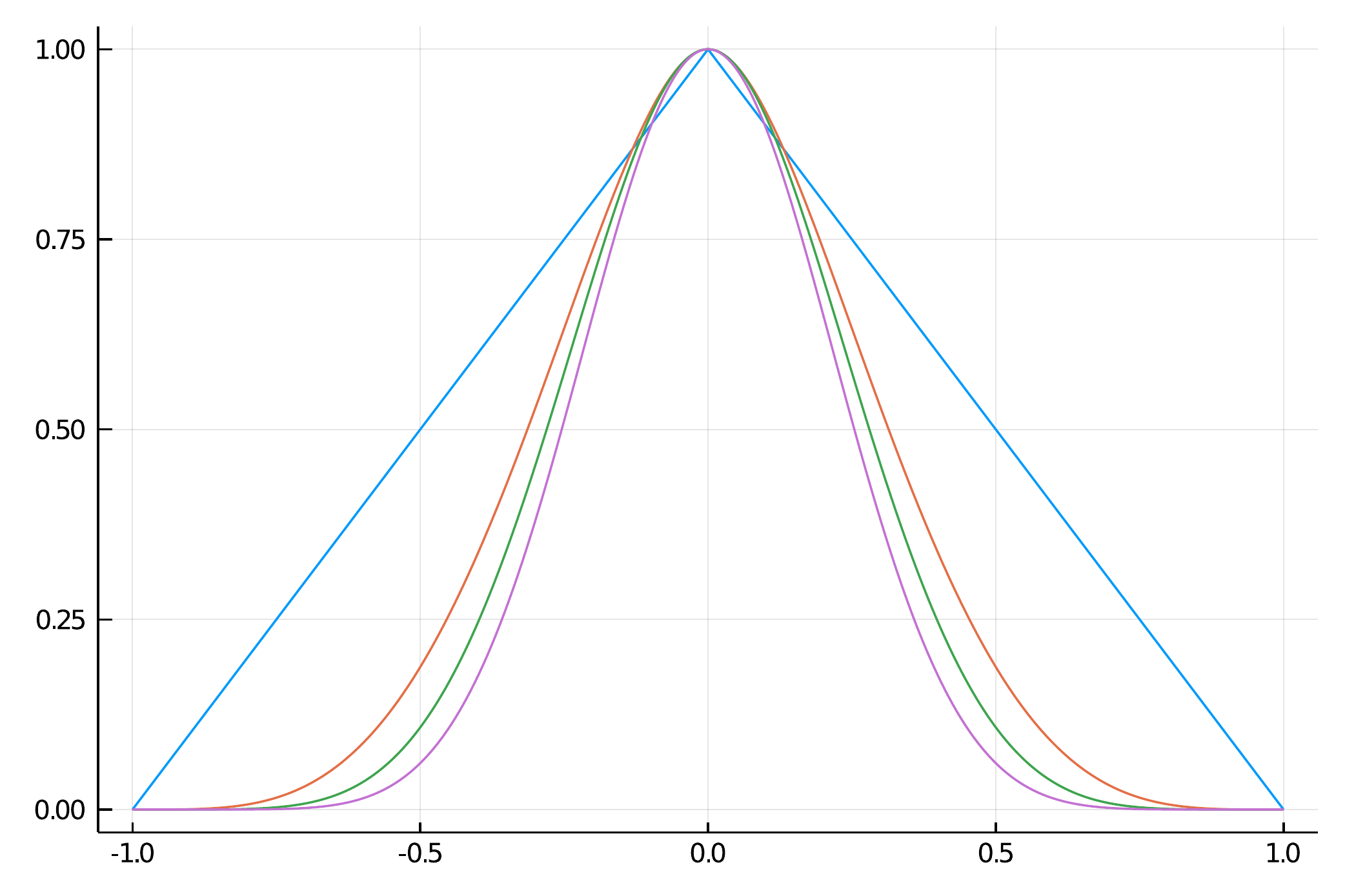}
         \caption{Some examples of Wendland polynomials}
        \label{fig:wendland_fns}
     \end{subfigure}
    \begin{subfigure}[t]{0.33\textwidth}
        \centering
        \includegraphics[width=\linewidth]{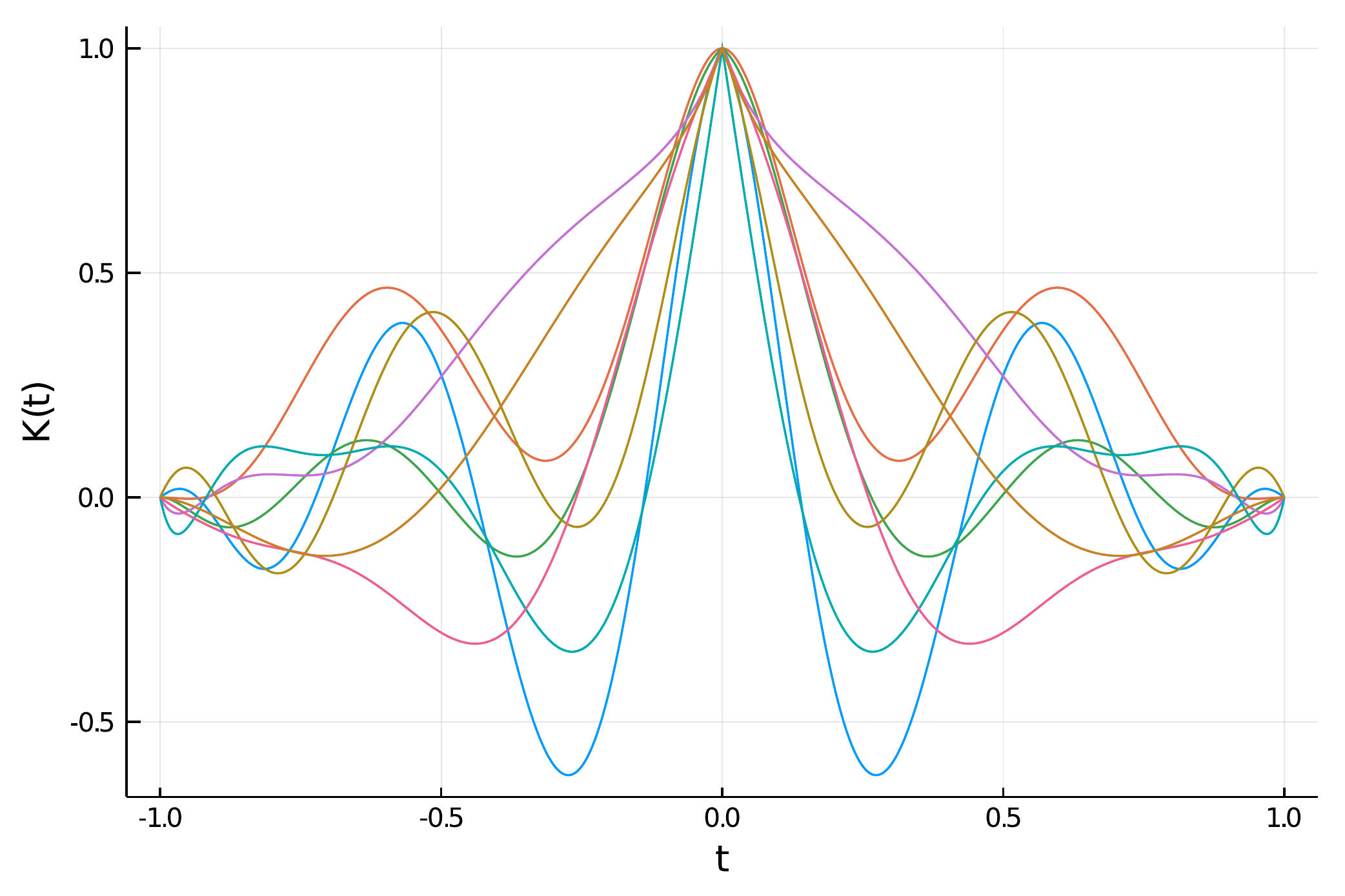}
        \caption{Polynomials (order 5)}
        \label{fig:poly_rand}
    \end{subfigure}%
    \begin{subfigure}[t]{0.33\linewidth}
        \centering
        \includegraphics[width=\linewidth]{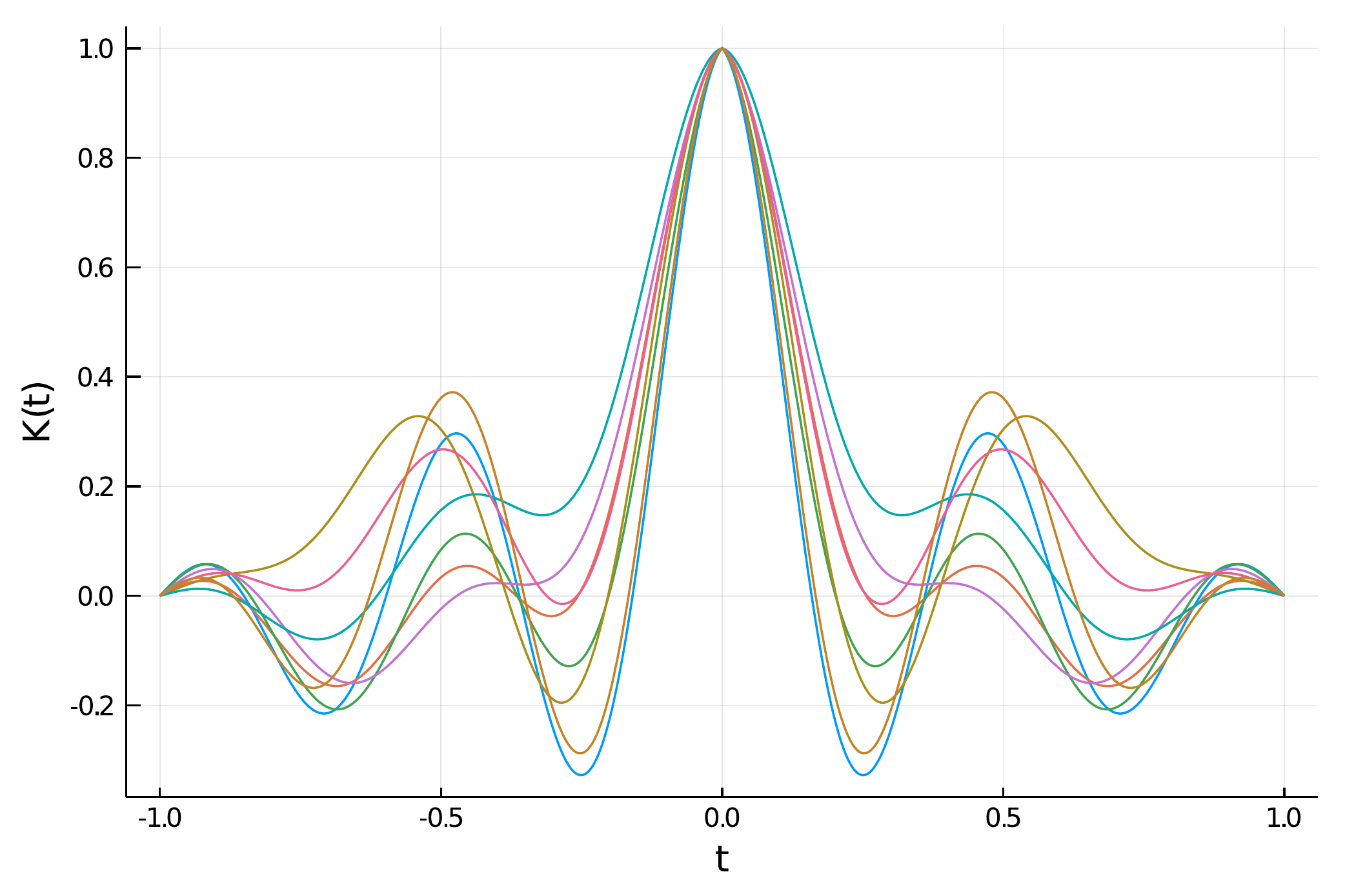}
        \caption{Fourier modes (order 3)}
        \label{fig:fourier_rand}
    \end{subfigure}
    \caption{Examples of compactly supported kernel functions. (a) previously described functions (Wendland polynomials) vs. parametric kernels generated from (b) polynomials and (c) Fourier modes.}
    \label{fig:random_functions}
\end{figure}
\section{Background}
\textbf{Gaussian processes and positive definite kernels}
Gaussian processes can be thought of as distributions over functions $f : \mathbb{R}^d\rightarrow \mathbb{R}$ with the property that, for any finite collection of $N$ points $x_k \in \Rn^d$, the function values $f(\vf x_1)\cdots f(\vf x_N)$ are distributed normally:
\begin{equation}
    \log p(\mathbf{y}) = -\frac{1}{2}(\mathbf y - \mathbf \mu)^T\mathbf K^{-1}(\mathbf y - \mathbf \mu) - \frac{1}{2}\log |2\pi \mathbf K|
    \label{eq:gp_ll}
\end{equation}
where $\vf y_i = f(\vf x_i)$, $\mu_i = \mu(x_i)$ for some mean function $\mu : \Rn^d \rightarrow \Rn$, and $\mathbf K_{ij} = K(\vf x_i, \vf x_j)$ for some positive definite kernel function $K$. For the rest of this paper, we assume that $\mu(x) = 0$; extending to non-zero $\mu$ is fairly trivial.

A typical use case for Gaussian processes is to observe some data $(\vf x_i, f(\vf x_i))_{i=1\cdots N}$, then compute the posterior distribution over a set of unobserved points; by fixing $M$ points $\vf x_1'\cdots \vf x_M'$, the GP posterior is given by:
\begin{align}
    \vf y' &= (f(\vf x_1') \cdots  f(\vf x_M')) \sim \mathcal{N}\left(\mathbf{K}'\mathbf{K}^{-1}\mathbf{y},
   \mathbf{K}'' - \mathbf{K}'\mathbf{K}^{-1}\mathbf{K}'^{T} 
    \right)
    \label{eq:gp_inference}
\end{align}
where $\mathbf{K}'_{ij} = K(x'_i, x_j)~\text{and}~ \mathbf{K}''_{ij} = K(x'_i, x'_j)$. The restrictions on the kernel function $K(\vf x, \vf x')$ is positive definiteness; that is, that for any finite collection of points $\mathbf{x}_i \in \Rn^d$, the kernel matrix $\mathbf{K}$ is positive semi-definite.  If we assume that the kernel matrix is \emph{translation invariant}, then we can write 
$K(\mathbf{x}_i, \mathbf{x}_j)$ as $K(\mathbf{x}_i - \mathbf{x}_j)$. These kernels can be analyzed using Fourier transforms, and positive definiteness is equivalent to the kernel function $K(\mathbf x)$ having a real, non-negative Fourier transform (Bochner's theorem\cite{williams2006gaussian}).

Generally, a modeler will choose a family of kernels parameterized by a set of parameters $\theta$. The standard method for estimating the kernel parameters from a set of training data is by maximizing the log-likelihood \eqref{eq:gp_ll} using gradient-based methods. The gradient of \eqref{eq:gp_ll} w.r.t. a kernel parameter $\theta_i$ can be computed using matrix calculus as:
\begin{equation}
    \partial_{\theta_i} \log p(\mathbf y) =
    \frac{1}{2}\mathbf{y}^T\mathbf{K}^{-1}(\partial_{\theta_i} \mathbf{K})\mathbf{K}^{-1}\mathbf y - \frac{1}{2}\tr{\mathbf{K}^{-1}\partial_{\theta_i}\mathbf{K}}
    \label{eq:ll_grad}
\end{equation}

\textbf{Iterative and stochastic methods}
Iterative methods (also referred to as ``matrix-free methods'') solve a linear system $\vf A \vf x =\vf b$ without requiring a numerical representation of $\vf A$, but by assuming the existence of a black-box algorithm for computing the product $\vf A\vf x$ on an arbitrary input $\vf x$. The prototypical example of this type of method is the \emph{conjugate gradient} (CG) algorithm\cite{shewchuk1994introduction}. CG poses the problem of solving $\vf x = \vf A^{-1}\vf b$, where $\vf A$ is an $N\times N$ positive-definite matrix, as a convex minimization problem:
\begin{equation}
    \vf x = \underset{\vf x'}{\operatorname{argmin}} ~\| \vf A \vf x' - \vf b\|^2
    \label{eq:cg_min}
\end{equation}
Then, starting from an initial guess $\vf x_0$, it takes a sequence of optimally-sized steps $\vf p_1, \vf p_2, \cdots$ along the gradient vectors of \eqref{eq:cg_min}, projected to be conjugate to previous step directions:
\begin{equation*}
    \vf p_i^T \vf A\vf p_j = 0 ~(i\neq j)
\end{equation*}
CG is guaranteed to converge in at most $N$ steps; since a generic MVM requires $O(N^2)$ time, this yields $O(N^3)$ time for computing $\mathbf{A}^{-1}\mathbf{b}$. In practice, CG reaches reasonable accuracy (such as floating point precision) in much fewer iterations; the exact convergence rate depends on the condition number of the matrix.

While scalable methods have been developed to compute the log determinant term in $\eqref{eq:gp_ll}$\cite{dong2017scalable}, we are more interested in the gradient \eqref{eq:ll_grad} with respect to a parameter $\theta$. The first term in \eqref{eq:ll_grad} is the gradient of the log determinant. Given a fast algorithm for computing $\mathbf{K}^{-1}\mathbf b$ (for some $\mathbf{b}$), we can construct a stochastic estimator for the trace term\cite{hutchinson1989stochastic}\cite{dong2017scalable}:
\begin{equation}
     \tr{\mathbf{K}^{-1}\partial_{\theta_i} \mathbf{K}} =
     \mathbb{E}_\mathbf{b}\left[\langle\mathbf{K}^{-1}\mathbf{b}, (\partial_\theta \mathbf{K})\mathbf{b}\rangle\right],~\mathbf{b}\sim\mathcal{N}(0, I)
\end{equation}
Replacing the trace term in \eqref{eq:ll_grad} yields an unbiased gradient estimator, and we can perform SGD without directly computing $\vf K^{-1}\partial_{\theta_i}\vf K$.

\textbf{Sparse linear algebra}
The particular matrix structure we seek to exploit is \emph{sparsity}. If an $N \times N$ matrix has only $N_\text{NZ}$ non-zero entries (with $N \le N_\text{NZ} \le N^2$), then computing a matrix-vector product requires $O( N_\text{NZ})$ time and $O( N_\text{NZ})$ space (since we only need to store the non-zero values along with their indices). If we can construct reasonably sparse ($N_\text{NZ}$ scaling as $O(N)$) kernel matrices, and they are well enough conditioned that CG converges sufficiently in $N_c \ll N$ iterations, then we expect to see sub-quadratic time matrix inversions. Whether that is achievable or not is highly dependent on the data; we have found this works well with time series data due to the distribution of distances between points having a low coefficient of variation (standard deviation relative to mean).

Another complexity concern is the computation of the kernel matrix.  While this is a lower order term (asymptotically) than inversion or determinant calculation, it can still be significant. In the dense case with standard linear algebra, the full $N\times N$ kernel matrix needs to be computed. In the sparse case, we only need to compute the $N_\text{NZ}$ non-zero entries of the kernel matrix. However, we need to know which entries are non-zero; na\"ively, this can be can done in $O(N^2)$ time. If the data is one dimensional and pre-sorted, we can do this in $O(N)$ time using a sliding window technique; higher dimensions can be handled with orthogonal range reporting techniques\cite{van2000computational}.

\textbf{Compactly supported kernel functions}
In general, kernel matrices are never naturally sparse. While certain kernels such as the square-exponential kernel $K(t) = e^{-t^2}$ decay very rapidly, they never reach zero. To achieve sparsity, we can try to find kernel functions that have compact support; i.e., they are zero outside of some interval $[-c, c]$ (which we will generically take to be $[-1, 1]$).  While any function can be converted to have compact support by truncation (i.e., define $\tilde f(t)$ to be equal to $f(t)$ for $t \in [-1, 1]$ and zero outside of that set), the result will usually not be valid kernels. 

Some compactly supported positive definite kernels have been identified in the literature, such as the Wendland polynomials \cite{wendland1995piecewise} and the function described in \cite{melkumyan2009sparse}. The most well-known subset of the former are shown in Figure \ref{fig:wendland_fns}; these are defined by:
\begin{equation}
    \begin{aligned}
    w_1(t) &= (1 - |t|)_+ &
    w_2(t) &= (1 - |t|)_+^4(4|t|+1) \\
    w_3(t) &= \frac{1}{3}(1 - |t|)_+^6(35|t|^2 + 18|t| + 3) & 
    w_4(t) &= (1 - |t|)_+^8(32|t|^3 + 25|t|^2 + 8|t| + 1) 
    \end{aligned}
    \label{eq:wendland_polys}
\end{equation}
where $(x)_+$ is the ReLU function. As we can see from the plots, the various compactly supported kernels have qualitatively ``Gaussian-like'' shapes; this observation is formalized in \cite{chernih2014wendland}, where the Wendland polynomials are shown to converge to a squared-exponential kernel in appropriate limits (with the region of support going to infinity).
\section{Parametric families of compactly-supported kernels}
Our goal is to generalize and extend these functions to larger families of compact kernels. The main insight that we use is the fact that for an arbitrary function $f(t)$, the \emph{autocorrelation function}:
\begin{equation}
    f_\text{corr}(t) = \int_\mathbb{R} f(x)f(x+t) dx
\end{equation}
is positive definite (whenever the above integral is defined); this follows from the convolution theorem in Fourier analysis and Bochner's theorem. We can take a linear combination of basis functions $f^\alpha(t) = \sum_k \alpha_k\phi_k(t)$, with $t$ restricted to the interval $[-1, 1]$, and autocorrelate it to result in a positive definite function $f_\text{corr}^\alpha(t)$. We can then take sums of these functions, for different coefficients $\alpha$, to form new postive definite functions. This process is formalized in the following main result, where all of the coefficients involved are contained in a positive semi-definite matrix:
\begin{theorem}
\label{thm:main}
Let $\{\phi_i\}_{i=1\cdots M}$ be a finite set of real- or complex-valued basis functions which are square integrable on $[-1, 1]$, i.e. $\int_{-1}^1 |\phi_i(t)|^2 dt < \infty$. 

Define the matrix-valued function $\Phi(t)$ as:
\begin{equation}
    \Phi_{ij}(t) = \frac{1}{2}\int_{-1}^{1 - \min\{|t|, 1\}} \phi_i^*(x)\phi_j(x+2|t|) + \phi_j(x)\phi_i^*(x+2|t|) dx
    \label{eq:phi_integral}
\end{equation}
Then, for any (real) $M\times M$ positive semi-definite matrix $\vf A$, the following function is a compact positive-definite kernel supported on $[-1, 1]$:
\begin{equation}
    K_A(t) = \tr{\vf A \Phi(t)}
\end{equation}
\end{theorem}

Given such a function $K_{\vf A}(t)$, we can extend its support to any closed interval $[-c,c]$ via rescaling $t$: $K_{\vf A,c}(t) = K_{\vf A}(t/c)$. We will refer to $c$ as the \emph{cutoff} to emphasize that it controls the support of the kernel, rather than its shape, and generally suppress it from notation.

It is natural to ask how many degrees of freedom the resulting parametric family of kernels possesses. For a fixed set of $M$ basis functions, the set of $M\times M$ positive semi-definite matrices is a smooth manifold of dimension $\frac{M(M+1)}{2}$, which provides an upper bound. A tight bound is determined by the following result:
\begin{theorem}
\label{thm:dim}
For strictly positive definite matrices $\vf A$, the image of the mapping $\vf A \mapsto \tr{\vf A\Phi(t)}$ is locally a smooth manifold of functions of dimension $r$, where $r$ is the number of linearly independent component functions of $\Phi_{ij}(t)$. 
\end{theorem}
In other words, the number of degrees of freedom is determined by the number of linearly independent component functions of $\Phi_{ij}$.
Proofs of these results are in Appendix \ref{sec:proofs}.

Parameter estimation is typically performed by maximizing data likelihood with gradient-based methods. We can compute the derivative of the kernel function $K_\vf A(t)$ with respect to the parameter $\vf A$ (the $\partial_\theta \mathbf K$ factor in \eqref{eq:ll_grad}) to be:
\begin{equation}
    \nabla_\vf A K_\vf A(t) = \Phi(t)
\end{equation}
\subsection{Concrete parametric families: polynomial and Fourier basis functions}
Natural choices of basis functions are those that make closed-form computation of the integrals in \eqref{eq:phi_integral} tractable. Notable examples are polynomials and exponentials. 

For example, choosing polynomials $\phi_{k}=t^k,~k=0\cdots(M-1)$ makes computation of the positive definite basis functions $\Phi_{ij}(t)$ simple to perform algorithmically via manipulation of polynomial coefficients, since polynomials are closed under integration, products, and composition. The result is that $\Phi(t)$ is a set of polynomials in $|t|$ of maximum degree $2M-1$.

Another natural choice of basis functions are complex exponentials of the form:
$$\phi_k(t) = \frac{1}{\sqrt{2}}e^{i\pi(k-1)t},~k \in 1\cdots M$$
These form the first $M$ components of a Fourier series decomposition of a function on $[-1,1]$. The correlation functions $\Phi(t)$ are:
\begin{equation}
    \Phi_{mn}(t) = \cos((m+n)\pi|t|)(1-|t|)\sinc\left[(n-m)(1-|t|)\right]
    \label{eq:fourier_corr_fns}
\end{equation}
where $\sinc(t) = \frac{\sin(\pi t)}{\pi t}$ is the normalized sinc function.
The final $\Phi$ functions are real-valued, so implementations do not require complex numbers. This basis is more parameter efficient, as the output function space dimensionality is the full $\frac{M(M+1)}{2}$ dimensionality of the parameter space, compared to a maximum dimensionality of $2M-1$ for the polynomial basis (in practice, the functions generated by the polynomial basis have even fewer degrees of freedom). In Figure \ref{fig:poly_rand} and \ref{fig:fourier_rand}, we sample kernels randomly from the polynomial and Fourier basis families to illustrate some of their qualitative characteristics.

These kernels are only valid in one dimesion, but can be extended to $\mathbb{R}^d$ via tensor products:
\begin{equation}
    K^d(\mathbf{x}) = \prod_{j=1}^d K(x_j)
\end{equation}
for a 1D kernel $K(t)$. If $K$ is compactly supported on $[-1, 1]$, $K^d$ will be compactly supported on the unit ball under the uniform norm, $[-1, 1]^d$.
\subsection{Compact approximations of non-compact kernels}
While the kernel parameter matrix can be learned directly from data by maximizing the likelihood function, a primary application of these compact kernel functions is to approximate other kernels of interest, such as Squared Exponential or Mat\'ern kernels. Truncating the kernels at some point will, in general, \emph{not} result in a well-defined positive definite kernel; this can lead to pathologies like singular kernel matrices or negative variances (when the kernel matrix has negative eigenvalues). To fix this, we can consider fitting a compact kernel directly to a target kernel $K(t)$ in a efficient manner. Due to the structure of the parameterization from Theorem \ref{thm:main}, this fitting procedure takes the form of a convex optimization problem.

The approach we take is to minimize the $L_2$ norm between the kernel functions.  The $L_2$ norm is less motivated as a loss than, say, an information divergence between the induced Gaussian processes, but it is much more tractable to compute. Assuming that we want a compact approximation with support on the interval $[-c, c]$, we want to minimize: 
\ifpreprint
    \begin{align}
        \mathcal{L}(\vf A) &=\frac{1}{2}\| \tr{\vf A\Phi(t)} - K(t) \|^2_{L_2([-c,c])} 
        \\&= \frac{1}{2}\int_{-c}^c \left(\tr{(\vf A\Phi(t/c))} - K(t)\right)^2 dt
    \end{align}
\fi
We can expand this out in terms of inner products on the Hilbert space $L_2([-c,c])$ (abbreviating our approximate kernel as $\tilde K$):
\ifpreprint
    \begin{align}
        \mathcal{L}(\vf A) &=\frac{1}{2}\langle \tilde{K}, \tilde{K} \rangle - \langle \tilde{K}, K \rangle + \text{~const.} 
        \\ &= \frac{1}{2}\sum_{ijkl} R_{ijkl}\vf A_{ij}\vf A_{kl} - \sum_{ij}B_{ij}\vf A_{ij} + \text{~const.}
    \end{align}
    where the tensors $R$ and $B$ are defined by:
    \begin{align}
        R_{ijkl} &= \langle \Phi_{ij}, \Phi_{kl} \rangle = 2\int_0^c \Phi_{ij}(t/c)\Phi_{kl}(t/c) dt
        \\B_{ij} &= \langle \Phi_{ij}, K \rangle = 2\int_0^c K(t) \Phi_{ij}(t/c) dt 
    \end{align}
\fi
The above integrals are fixed numeric values, and can be computed offline through numerical integration.  We found it to be helpful to enforce a peak matching constraint; that is, $\tr{\vf A\Phi(0)} = K(0)$, which amounts to a linear constraint on $A$. This effectively biases the solution to fit better near the origin, which seems to improve empirical results.

Taken together, these conditions define the following convex optimization problem:
\begin{align*}
    \textrm{Minimize~} &~\frac{1}{2}\sum_{ijkl} R_{ijkl}\vf A_{ij}\vf A_{kl} - \sum_{ij}B_{ij}\vf A_{ij}
    \\
    \textrm{Subject to:~}
    &\vf A \in \mathcal{S}_+^{M} \\
    &\tr{\vf A\Phi(0)} = K(0)
\end{align*}
We solve this optimization problem using the COSMO\cite{garstka2019cosmo} algorithm. The entire process (computing the $R_{ijkl}$ and $B_{ij}$ tensors and solving the convex problem) takes a few seconds on a modern laptop, and is completely amortized outside of GP inference.
\section{Experiments}
We analyze the errors associated with compact approximations of several kernel functions of interest:
\begin{align}
    K_\text{SE}(t) &= e^{-t^2} &\text{(Squared-exponential)} \label{eq:se_kern}\\
    K_\text{OU}(t) &= e^{-|t|} &\text{(Ornstein–Uhlenbeck)} \\
    K_\text{Mat\'ern}(t) &= e^{-\sqrt{5}|t|}\left(1 + \sqrt{5}|t| + \frac{5}{3}t^2\right)&\text{(Mat\'ern 5/2)} \\
    K_\text{sinc}(t) &= \frac{\sin (\pi t)}{\pi t}  &\text{(Sinc/Band limited)}\label{eq:sinc_kern}
\end{align}
The first three are common kernels in GP literature; the fourth, the sinc kernel, has been suggested in \cite{tobar2019band} as an appropriate kernel for applying Gaussian processes to signal processing tasks. This kernel seems challenging to approximate, due to its $1/|t|$ decay rate inducing much longer-range correlations than other common kernels.

\textbf{Compact approximation errors} We show compact approximations of these kernels for a constant cutoff parameter value of 5 in Figure \ref{fig:approx}. Kernels generated from the Fourier basis generally perform better, likely due to having more degrees of freedom at a given order. The exception is the Ornstein-Uhlenbeck kernel, where it is unable to replicate the sharp peak at zero.
\begin{wrapfigure}{L}{0.35\linewidth}
    \centering
    \includegraphics[width=\linewidth]{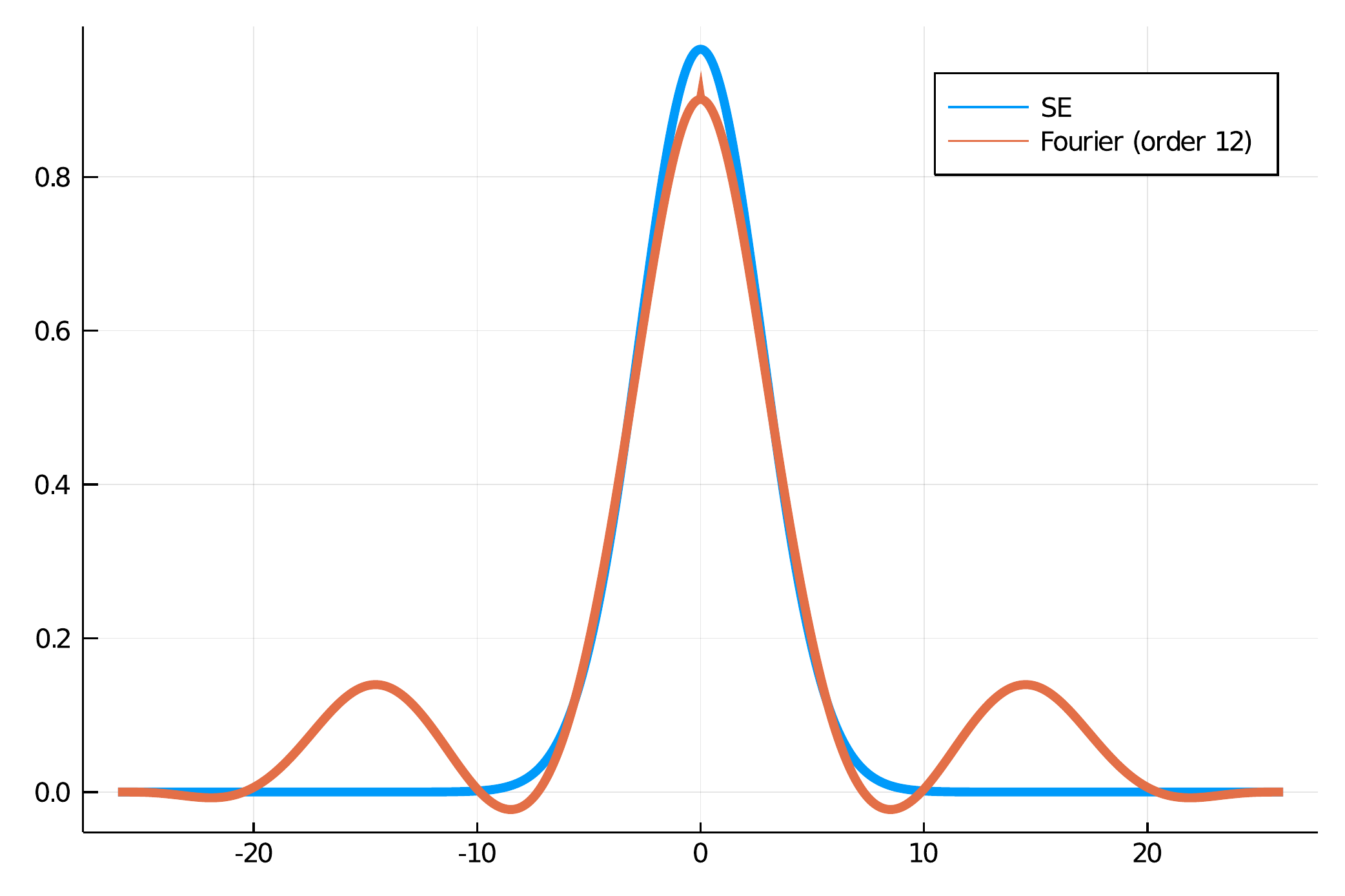}
    \caption{Best non-compact kernel (SE) and best optimized compact kernel when trained on TIMIT data.}
    \label{fig:best_kern}
\end{wrapfigure}

\begin{figure}
    \centering
    \begin{subfigure}[t]{0.24\linewidth}
    \centering
    \includegraphics[width=\linewidth]{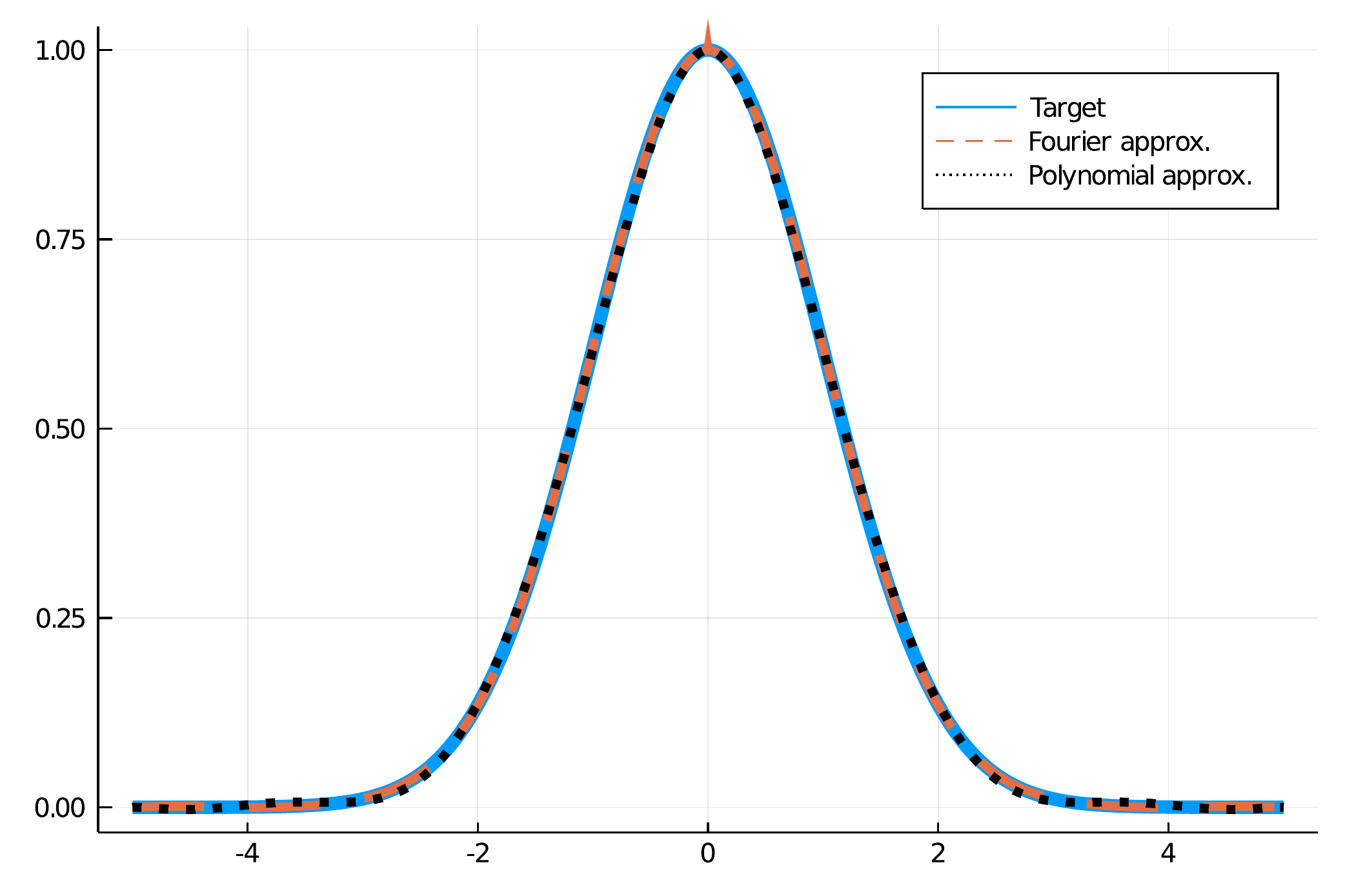}
    \caption{\tabular[t]{@{}l@{}}SE kernel \\ $\epsilon_F$=7.3E-6\\$\epsilon_P$=2.9E-4\endtabular}
    \end{subfigure}
    \hfill
    \begin{subfigure}[t]{0.24\linewidth}
    \centering
    \includegraphics[width=\linewidth]{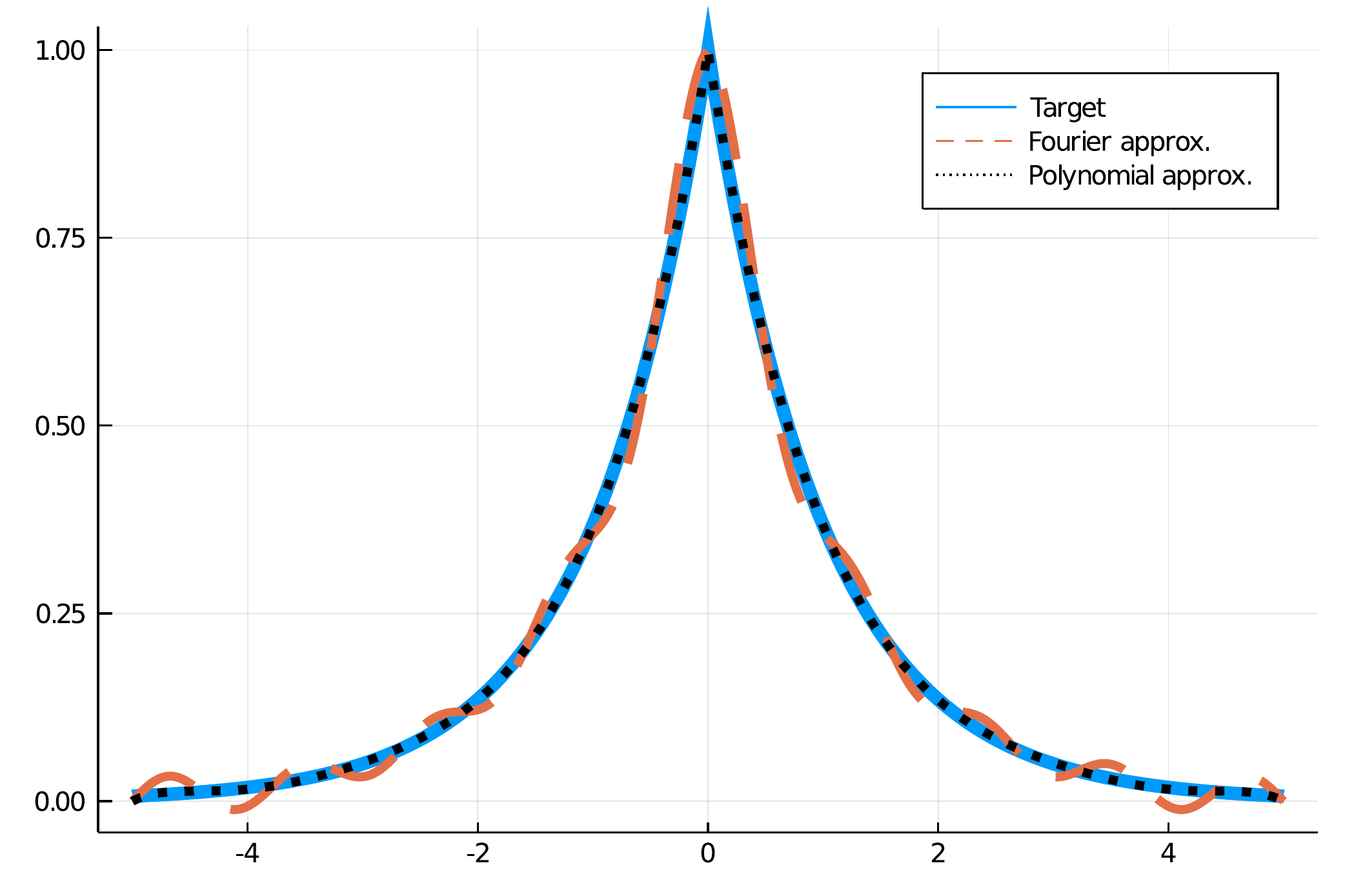}
    \caption{\tabular[t]{@{}l@{}}OU kernel \\ $\epsilon_F$=6.1E-4\\$\epsilon_P$=3.3E-5\endtabular}
    \end{subfigure}
    \begin{subfigure}[t]{0.24\linewidth}
    \centering
    \includegraphics[width=\linewidth]{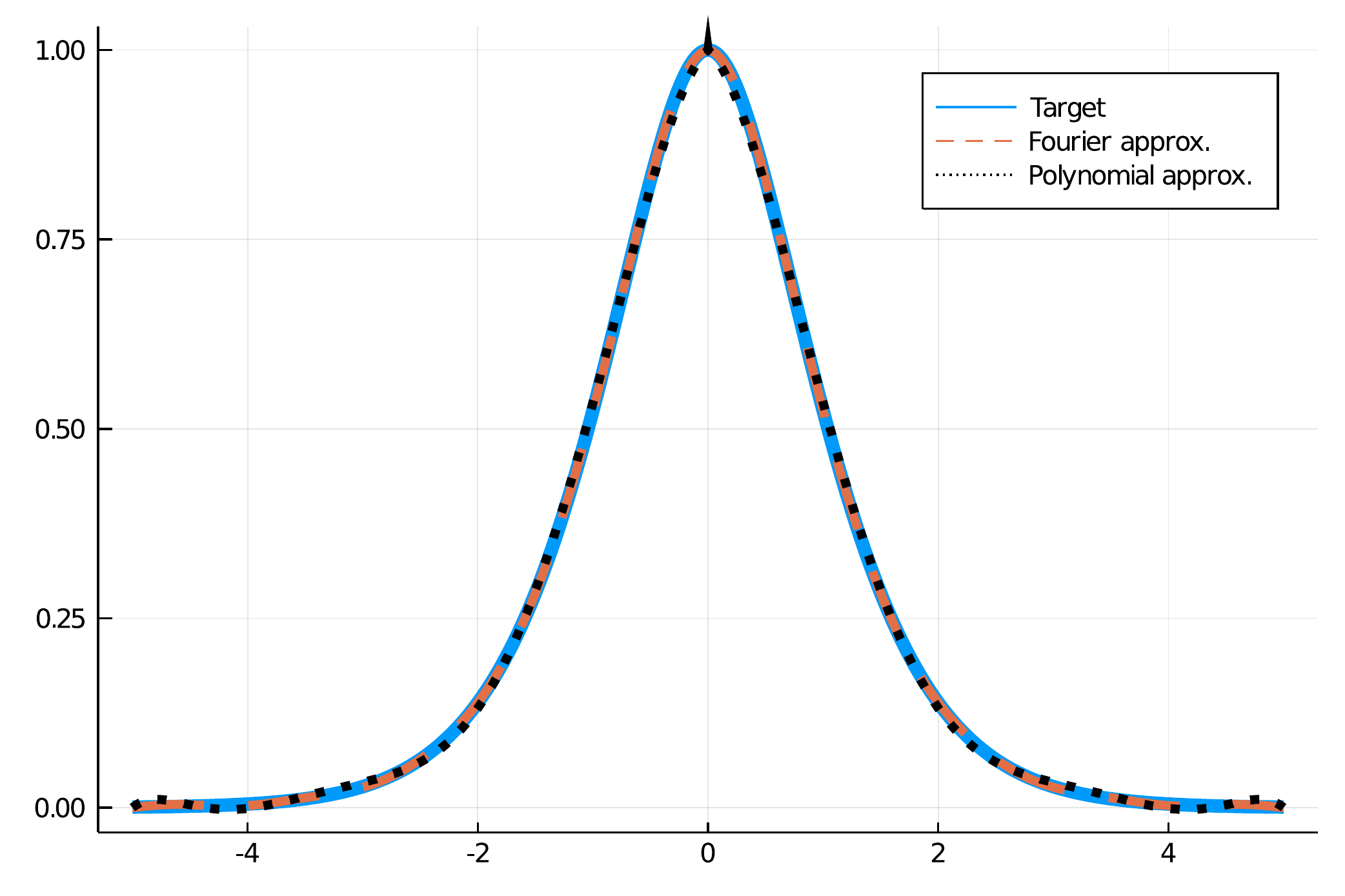}
    \caption{\tabular[t]{@{}l@{}}Mat\'ern kernel \\ $\epsilon_F$=1.1E-5\\$\epsilon_P$=5.1E-4\endtabular}  
    \end{subfigure}
    \hfill
    \begin{subfigure}[t]{0.24\linewidth}
    \centering
    \includegraphics[width=\linewidth]{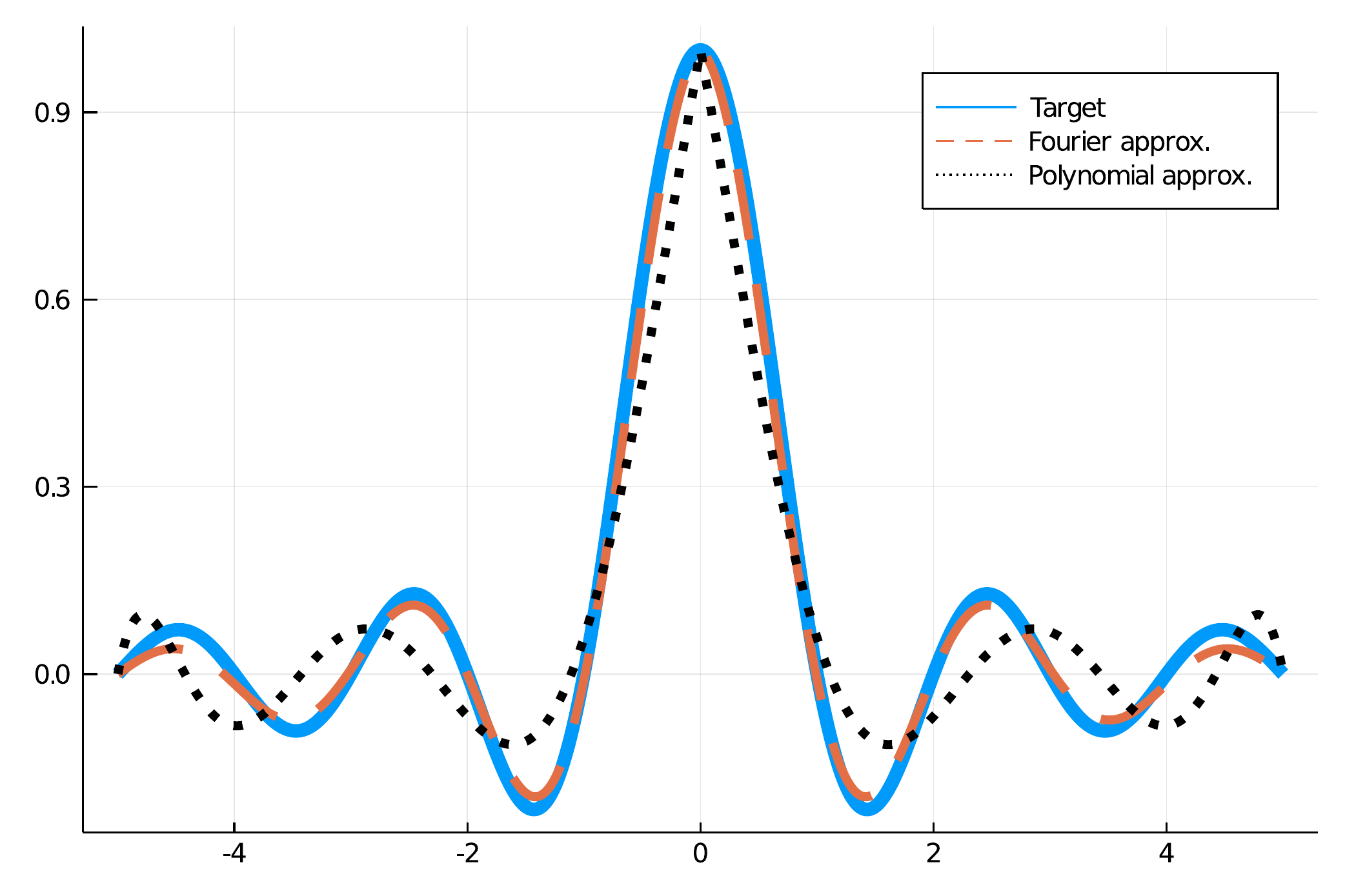}
    \caption{\tabular[t]{@{}l@{}}Sinc kernel \\ $\epsilon_F$=4.0E-3\\$\epsilon_P$=8.0E-2\endtabular} 
    \end{subfigure}
    \caption{Least-squared approximations of kernel functions \eqref{eq:se_kern}-\eqref{eq:sinc_kern} for order 5 compact parametric kernels. $\epsilon_F, \epsilon_P$ are $L_2$ errors for Fourier and polynomial basis kernels, respectively.}
    \label{fig:approx}
\end{figure}
In order to evaluate how well compact kernels approximate fixed kernels, we perform experiments on data sampled from known kernels. In Figure \ref{fig:rmse_nll_synth}, we generate random problems from the four kernels \eqref{eq:se_kern}-\eqref{eq:sinc_kern} with 1024 training instances. We then compare training NLL scores as well as predicted RMSE on a hold-out test set of the original kernel along with a compact approximation, fit using the Fourier basis functions \eqref{eq:fourier_corr_fns}. We can see that, as expected, we pay a model mismatch price both in terms of RMSE and NLL. The sinc kernel had by far the largest approximation error, which is unsurprising due to its slow decay rate.
\begin{figure}[b]
    \centering
    \begin{subfigure}[t]{0.45\linewidth}
    \includegraphics[width=\linewidth]{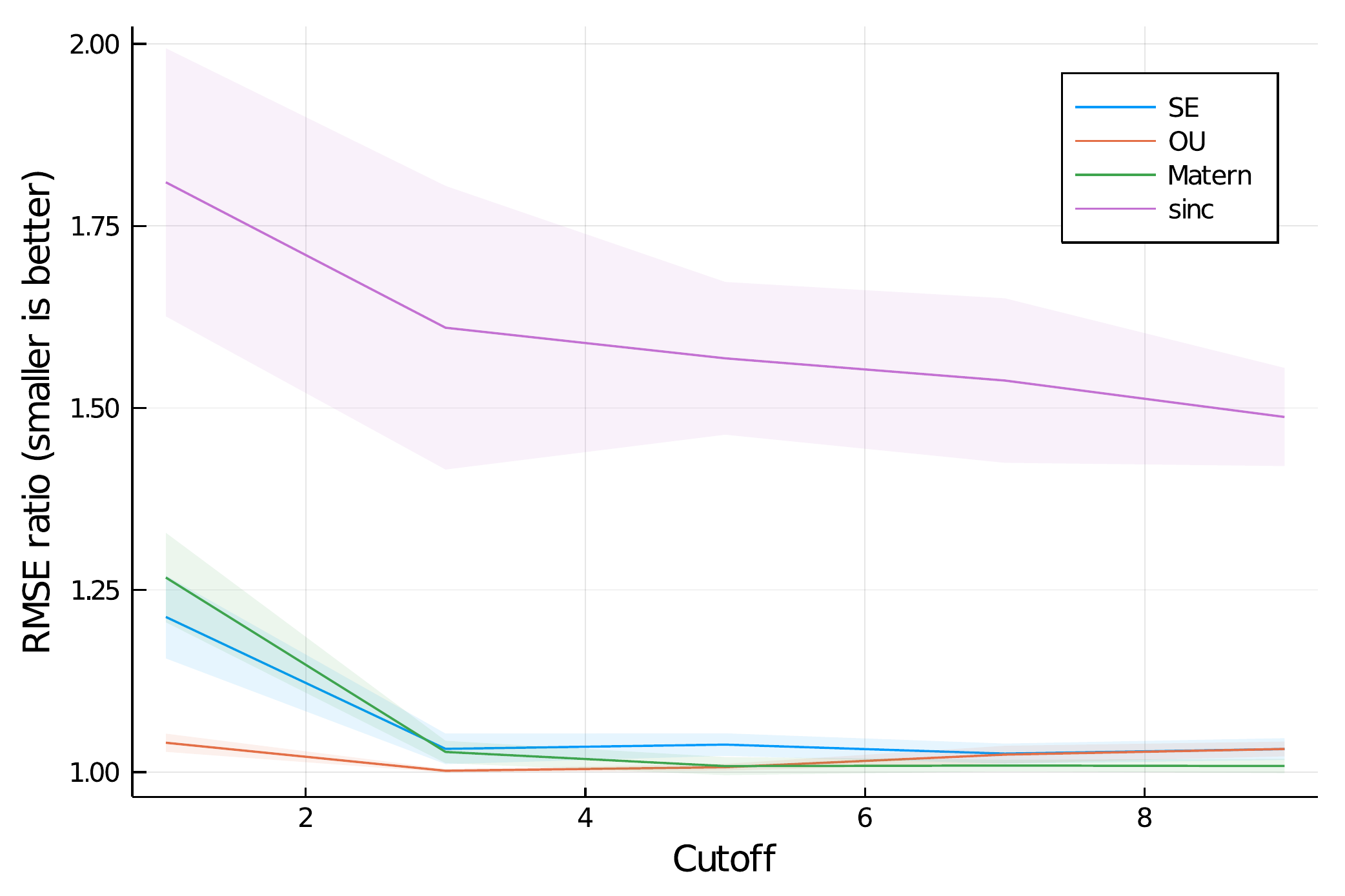}
    \caption{Test RMSE ratios}
    \end{subfigure}
    \hfill
    \begin{subfigure}[t]{0.45\linewidth}
    \includegraphics[width=\linewidth]{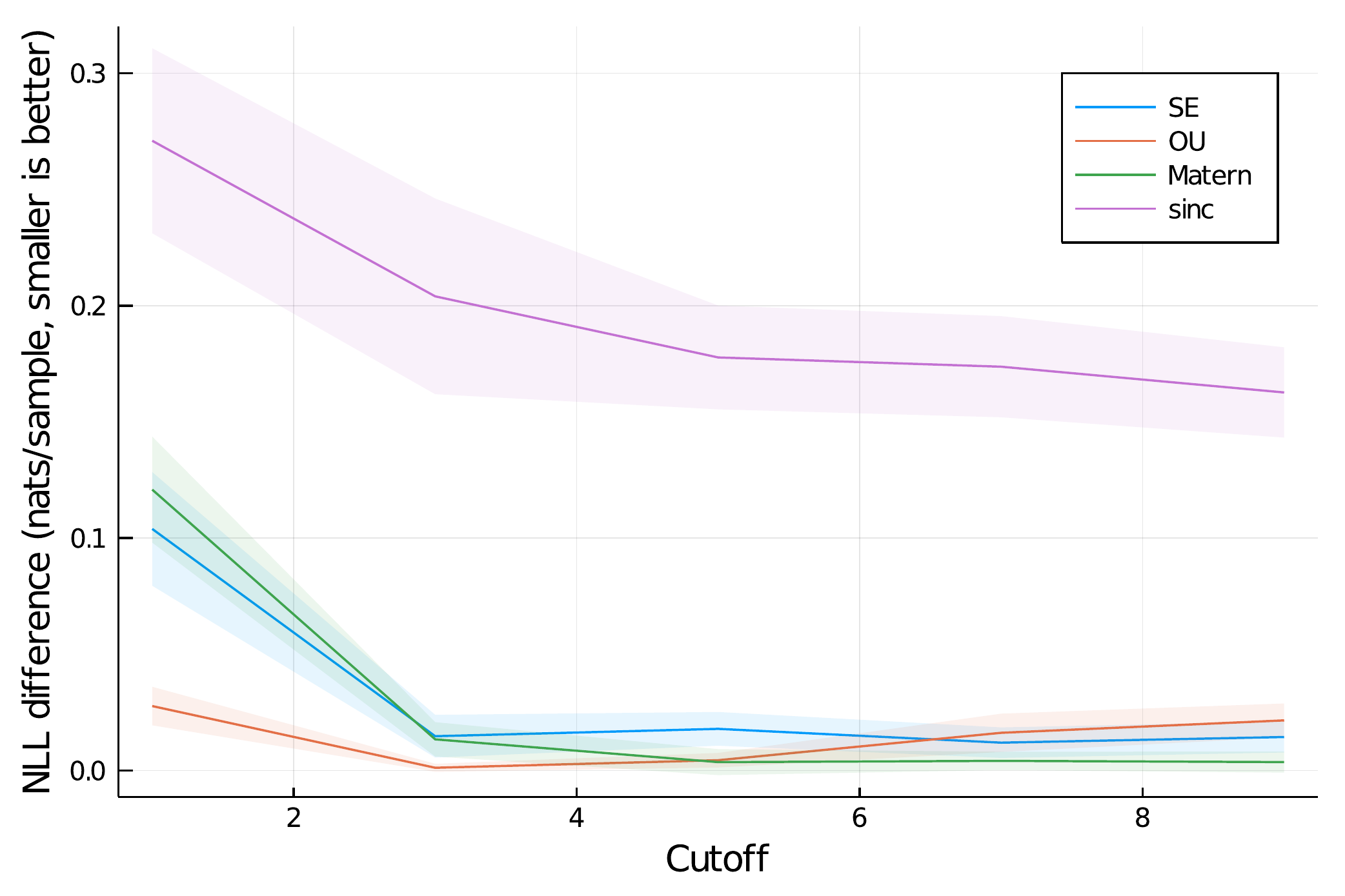}
    \caption{Training NLL differences}
    \end{subfigure}
    \caption{Approximation errors from modeling data sampled from known kernels with compact approximations of those kernels.}
    \label{fig:rmse_nll_synth}
\end{figure}

\textbf{Performance gains} In Figure \ref{fig:performance}, we compare the run time performance of posterior mean inference as a function of number of training points. We generated random GP inference problems with number of training examples ranging from 64 to 25000. We then timed the posterior mean calculation using both dense and sparse linear algebra. Not included in the time calculation was the kernel matrix formation (in the dense case) or the sparsity pattern calculation (in the sparse case); these both dominate in the low $N$ regime, so including them would make it more difficult to measure how the runtime scales. Additionally, the sparsity pattern calculations are amortized away when performing tasks like gradient-based likelihood maximization which can re-use the same sparsity pattern for each gradient calculation step.

\begin{figure}
    \centering
    \begin{subfigure}[t]{0.48\linewidth}
        \centering
        \includegraphics[width=\linewidth]{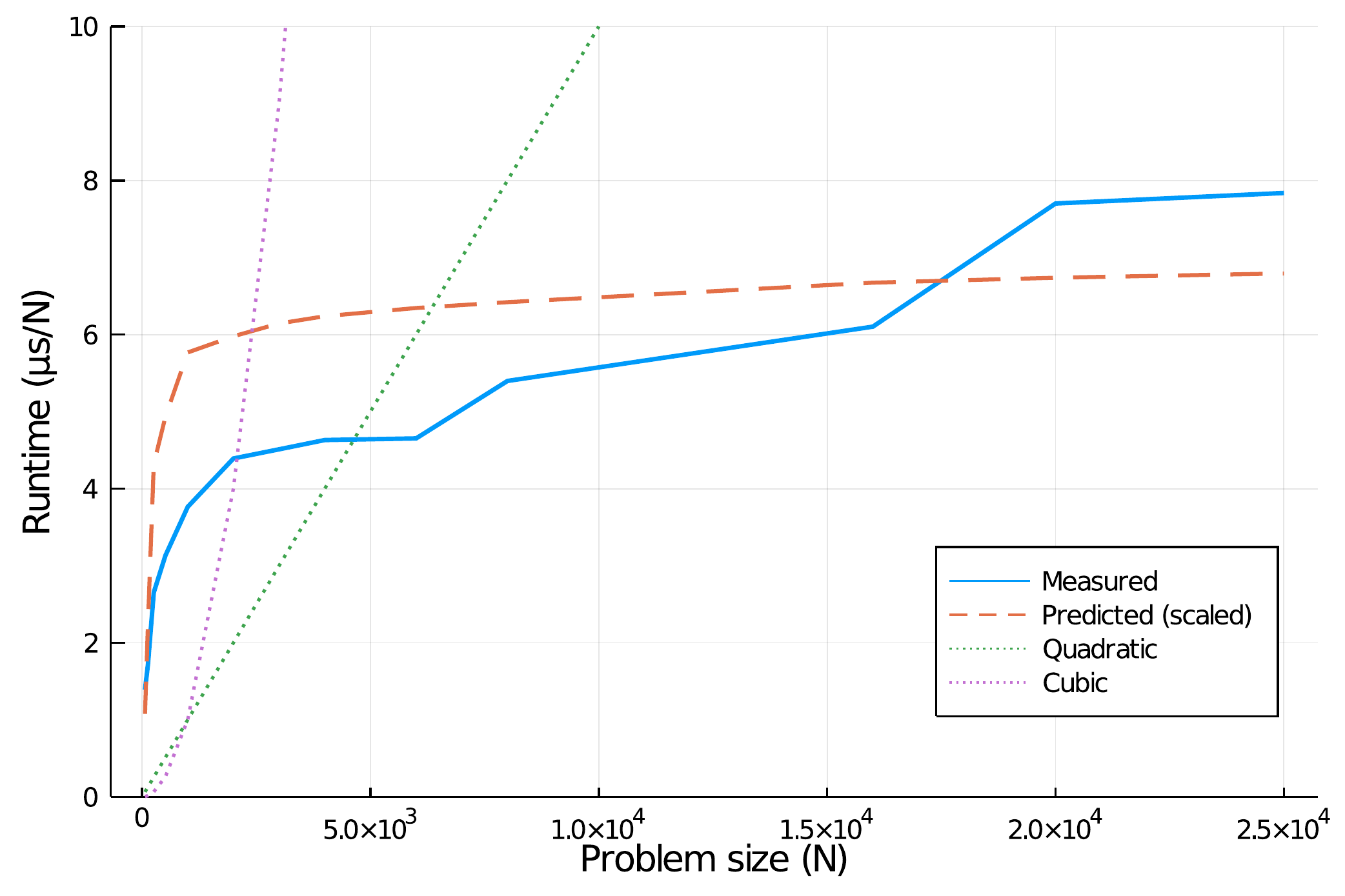}
        \caption{}
        \label{fig:perf_a}
    \end{subfigure}
    \hfill
    \begin{subfigure}[t]{0.48\linewidth}
        \centering
        \includegraphics[width=\linewidth]{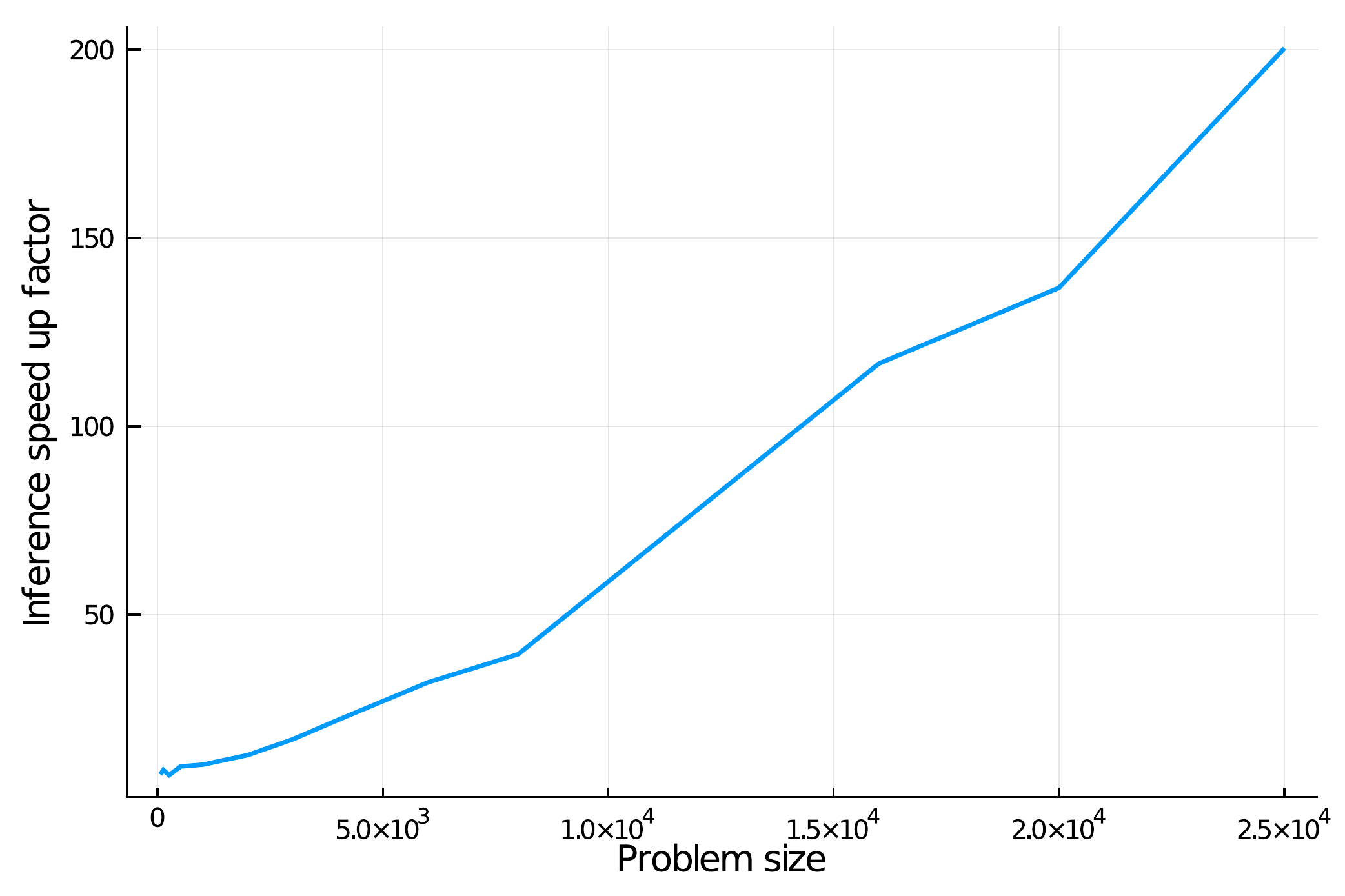}
        \caption{}
        \label{fig:perf_b}
    \end{subfigure}
        \caption{(a) Scaling characteristics of posterior mean inference with compact kernels vs. non-compact kernels for synthetic time series data. Values are scaled by $1/N$ to illustrate scaling behavior. The ``predicted'' line is proportional to $\text{(\# of non-zero kernel entries)(\# of CG iterations)}/N$, which is the number of multiplies required to invert the kernel matrix with CG. (b) Empirical runtime ratio of dense inference vs. sparse inference.}
    \label{fig:performance}
\end{figure}
\begin{table}[]
    \centering
    \small
    \begin{tabular}{lcccccc}
    \toprule
    {} & \multicolumn{3}{c}{Initial} & \multicolumn{3}{c}{Tuned}
    \\ \cmidrule(lr){2-4} \cmidrule(lr){5-7}
    Kernel & Train NLL & Test NLL & Test RMSE & Train NLL & Test NLL & Test RMSE
 \\\midrule
 SE & 0.86±0.037 & 0.86±0.037 & 0.34±0.022 & 0.25±0.1 & 0.25±0.11 & 0.14±0.014
 \\ SE (approx) & 0.81±0.025 & 0.81±0.025 & 0.27±0.018 & 0.3±0.086 & 0.3±0.086 & 0.16±0.016
 \\\hline OU & 0.77±0.04 & 0.77±0.04 & 0.27±0.019 & 0.73±0.05 & 0.73±0.05 & 0.26±0.019
 \\ OU (approx) & 0.59±0.03 & 0.59±0.029 & 0.21±0.016 & 0.62±0.052 & 0.62±0.015 & 0.23±0.018
 \\\hline Matern & 0.5±0.034 & 0.5±0.034 & 0.19±0.014 & 0.27±0.074 & 0.27±0.074 & 0.14±0.014
 \\ Matern (approx) & 0.48±0.029 & 0.48±0.029 & 0.18±0.015 & 0.34±0.072 & 0.34±0.071 & 0.16±0.016
 \\\hline Sinc & 17±3 & 17±3 & 0.3±0.03 & 0.59±0.089 & 0.61±0.098 & 0.22±0.021
 \\ Sinc (approx) & 0.53±0.061 & 0.53±0.06 & 0.23±0.021 & 0.78±0.03 & 0.78±0.03 & 0.28±0.019
 \\\bottomrule
    \end{tabular}
    \caption{Comparison of non-compact kernels with their compact approximations, before and after tuning the kernel scale and length scale parameters of the target kernels.}
    \label{tab:timit}
\end{table}

\textbf{Real data}
We perform experiments on a dataset of 380 audio files from the TIMIT\cite{garofolo1993darpa} speech dataset using the kernels \eqref{eq:se_kern}-\eqref{eq:sinc_kern}, the Wendland polynomials in \eqref{eq:wendland_polys}, and our compact parametric kernels using the Fourier basis. Each audio file is pre-processed
, then 50\% of samples are assigned to a training set and 50\% to a test set.
 The four non-compact kernels are parameterized by kernel scale and length scale. The compactly supported kernels have a cutoff, so the Wendland kernels are parameterized by a single scale parameter and the Fourier basis kernels are parameterized by a positive semidefinite parameter matrix. We then compute NLL values on a held out test set as well as reconstruction RMSE. The results in Table \ref{tab:timit2} show that our parametric kernel families outperform both the Wendland polynomials as well as the baseline non-compactly supported kernels in terms of training NLL, testing NLL, and testing RSME. The best kernel found is shown in Figure \ref{fig:best_kern}; qualitatively, it is similar to the Lanczos functions, which are widely used for interpolation in image processing (although the Lanczos functions are not positive definite).
\begin{table}[b]
    \centering
    \small
    \begin{tabular}{lccc}
    \toprule
        Kernel & Train NLL & Test NLL & Test RMSE
        \\ \midrule SE (Best non-compact) & \textbf{0.25±0.1} & \textbf{0.25±0.11} & \textbf{0.14±0.014}
 \\ \hline $w_1$ \eqref{eq:wendland_polys} & 1.0±0.068 & 1.0±0.066 & 0.42±0.031
 \\ $w_2$ \eqref{eq:wendland_polys} & 0.53±0.11 & 0.52±0.11 & 0.17±0.016
 \\ $w_3$ \eqref{eq:wendland_polys} & 0.82±0.15 & 0.82±0.15 & 0.2±0.02
 \\ $w_4$ \eqref{eq:wendland_polys} & 0.85±0.15 & 0.86±0.15 & 0.21±0.021
 \\ \hline Fourier (order 3) & 0.51±0.069 & 0.5±0.067 & 0.23±0.022
 \\ Fourier (order 5) & \textbf{0.25±0.093} & \textbf{0.25±0.094} & 0.15±0.016
 \\ Fourier (order 8) & \textbf{0.2±0.071} & \textbf{0.19±0.068} & \textbf{0.13±0.013}
 \\ Fourier (order 12) & \textbf{0.19±0.089} & \textbf{0.19±0.09} & \textbf{0.13±0.013}
 \\ \bottomrule
    \end{tabular}
    \caption{Performance of compact kernels on TIMIT audio files after likelihood optimization. The best non-compact kernel (SE) from Table \ref{tab:timit} is shown for comparison. Significance indicated with a one-sided t-test with a 5\% p-value threshold.}
    \label{tab:timit2}
\end{table}
\section{Discussion}
In this work, we presented a generic method for generating parametric families of compactly-supported kernels, with two specific examples. These families can be fit to data directly, or used to find sparse approximations of existing translation-invariant kernels using convex optimization. The flexibility of the parameterization yields good fits to data, even compared to commonly used non-compact kernels which can model much longer range dependencies. Additionally, we showed empirically that for a typical time series-like task where the sparsity can be controlled (e.g., the number of non-zero kernel matrix entries grows as $O(N)$), we can perform kernel matrix inversion in sub-quadratic time with sparse linear algebra. 

We believe that these kernels are attractive models for low-dimensional applications such as time series analysis, signal processing, and modeling spatial data. These types of datasets typically have strong local correlations (due to the underlying features representing a physical quantity like time or space), which makes a compact kernel a reasonable inductive bias. They also tend to have relatively regular spacing between data points, which allows for very sparse kernel matrices and enables fast inference. We suspect that they will be less competitive when modeling higher dimensional data (such as general regression or classification problems on feature spaces). These applications have more complex feature space geometry, and the curse of dimensionality makes controlling the kernel matrix sparsity more challenging. However, there are many low-dimensional application areas where scalable GP inference is necessary to process large datasets; these include industries such as finance and geospatial modeling. Our hope is that the methods proposed in the paper, combined with fast GPU implementations such as GPyTorch\cite{gardner2018gpytorch}, can widen the adoption of GPs in these fields.
\begin{ack}
 This material is based upon work supported by the United States Air Force under Contract No. FA8750-19-C-0515. 
\end{ack}
\nocite{stein2012interpolation}
\nocite{lee2013smooth}
\bibliographystyle{unsrt}  
\bibliography{references}  
\appendix
\newpage
\section{Proofs}
\label{sec:proofs}
\subsection{Proof of Theorem 1}
First, we prove the following lemma:
\begin{lemma}
    Let $f$ be a measurable function $[-1, 1] \rightarrow \mathbb{C}$ Then, the function $h_f$ given by:
    \begin{equation}
        \label{eq:autocorr}
        h_f(t) = \intbounds f(x)^*f(x + 2|t|) dx
    \end{equation}
    is positive definite and supported on $[-1, 1]$.
\end{lemma}
\begin{proof}
    Given such an $f$, we can extend it to all of $\mathbb{R}$ by:
    \begin{equation}
        \tilde f(t) = \begin{cases} f(t) & |t| \le 1 \\
                                    0 & |t| > 1
                                    \end{cases}
    \end{equation}
    The autocorrelation of $\tilde f$ with itself is giv    n by:
    \begin{equation}
        \tilde h(t) = \int_{\mathbb{R}} \tilde f(x)^*\tilde f(x + t) dx
    \end{equation}
    Denoting $\tilde H(\omega)$ and $\tilde F(\omega)$ as the Fourier transforms of $\tilde h,~\tilde f$, respectively, we have:
    \begin{equation}
        \tilde H(\omega) = |\tilde F(\omega)|^2 \ge 0
    \end{equation}
    Thus, $h(t)$ is a positive definite function via Bochner's theorem\cite{stein2012interpolation}.
    Since $\tilde f$ is compactly supported, we can restrict the domain of integration to the support of the integrand: 
    \begin{align}
        \tilde h(t) &= \int_{\operatorname{max}\{|x|,|x + t|\} < 1}  f(x)^*f(x + t) dx
        \\
        &= \int_{-1}^{1-\min\{|t|,1\}} f(x)^*f(x + |t|) dx
    \end{align}
    It is now easy to check that $\tilde h$ is compactly supported on $[-2, 2]$. We re-scale $t$ to compress the support to $[-1, 1]$:
    \begin{equation}
        h_f(t) = \tilde h(2t) = \int_{-1}^{1-\min\{|t|,1\}}f(x)^*f(x + 2|t|)  dx
    \end{equation}
    Since positive-definiteness is preserved under constant rescaling of the domain, $h_f$ is a positive-definite function supported on $[-1, 1]$.
\end{proof}
Next, we can prove the main result:
\begin{proof}
    Let $\{\phi_i\}_{i=1\cdots N}$ be a finite collection of basis functions. If we consider a family of functions that are linear combinations of these basis functions (e.g., $f_\alpha(t) = \sum_i \alpha_i \phi_i(t)$), we can see that:
    \begin{align}
        K_\alpha(t) &= \int_{-1}^{1-\min\{|t|,1\}} \left(\sum_i \alpha_i \phi_i(x)\right)^*\left(\sum_i \alpha_i \phi_i(x+2|t|)\right) dx \\
        &= \sum_{i,j} \alpha_i^*\alpha_j \left(\int_{-1}^{1-\min\{|t|,1\}}\phi_i(x)^*\phi_j(x+2t) dx \right)\label{eq:sym_mat_ante}\\
        &= \sum_{i,j} \alpha_i^*\alpha_j \left(\frac{1}{2}\int_{-1}^{1-\min\{|t|,1\}}\left[\phi_i(x)^*\phi_j(x+2t) + \phi_i^*(x+2t)\phi_j(x)\right] dx \right)\label{eq:sym_mat}\\
        &= \alpha^H \Phi(t) \alpha\label{eq:Phi}
    \end{align}
    where $\Phi_{mn}(t)$ is given by the ``symmetric correlation'' of $\phi_m$ and $\phi_n$. The symmetrizing step \eqref{eq:sym_mat_ante} $\rightarrow$ \eqref{eq:sym_mat} is justified because the anti symmetric component of $\Phi$ does not contribute to the value of $K_\alpha$. Furthermore, since \emph{sums} of positive definite functions are positive definite, we can extend this family to sums of these, which are parameterized by a positive semi-definite (PSD) matrix $A$:
    \begin{equation}
        K_A(t) = \sum_{k} \left( \alpha_k^H\Phi(t)\alpha_k \right)  = \tr{A\Phi(t)},\text{~where~} A = \sum_k \alpha_k\alpha_k^H
    \end{equation}
    Here, $A$ can be either real (symmetric) or complex (Hermitian); if $\Phi(t)$ is real-valued, then only the real part of $A$ contributes to the value of $K_A(t)$.
\end{proof}
\subsection{Proof of Theorem \ref{thm:dim}}
\begin{proof}
Let $\mathcal{S}_N$ be the $m = \frac{N(N+1)}{2}$ dimensional vector space of $N\times N$ symmetric matrices, and let $V = \operatorname{span} \{ \Phi_{ij}(t) \}$, i.e. the $r$-dimensional vector space (with $r \le m$) of all possible linear combinations of the functions $\Phi_{ij}$. Then the mapping $\pi(X) = \sum_{ij}X_{ij}\Phi_{ij}(t)$ is a surjective linear map from $\mathcal{S}_N\rightarrow V$. We can think of $K_A$ as a smooth function from $\mathcal{S}^+_N \rightarrow V$ by embedding $\mathcal{S}^+_N$ in $\mathcal{S}_N$, and its differential (Jacobian) is given by:
\begin{equation}
    dK_A(dA) = \sum_{ij} dA_{ij} \Phi_{ij}(t)
\end{equation}
for a tangent vector $dA \in \mathcal{S}_N$. Since this is just the map $\pi$, it does not depend $A$ and therefore has constant rank $r$. Therefore, the proposition holds by the Rank Theorem of differential topology (see \cite{lee2013smooth}, Theorem 4.12).
\end{proof}
Intuitively, "Locally a smooth manifold of dimension $r$" means that for a fixed positive definite $A$, there are $r$ independent directions that one can move away from $A$ in the input space and generate different functions through the $A \mapsto \tr{A\Phi(t)}$ mapping.  Globally, the image could fail to be a manifold due to e.g. self-intersection or other pathologies.

\end{document}